\pdfoutput=1
\documentclass[english,sigconf]{acmart}

\renewcommand\footnotetextcopyrightpermission[1]{} %
\usepackage{enumitem}
\usepackage{amsmath, amssymb,bm, cases, mathtools, thmtools}
\usepackage{mathrsfs}

\usepackage{euscript,microtype} %

\usepackage[utf8]{inputenc} %
\usepackage[T1]{fontenc}    %
\usepackage{url}            %
\usepackage{booktabs}       %
\usepackage{amsfonts}       %
\usepackage{nicefrac}       %
\usepackage{bbm}
\usepackage{subfig}

\usepackage[capitalize]{cleveref}
\crefformat{equation}{(#2#1#3)}
\crefformat{figure}{Figure~#2#1#3}

\crefname{lemma}{Lemma}{Lemmas}
\crefname{cor}{Corollary}{Corollaries}
\crefname{theorem}{Theorem}{Theorems}
\crefname{assumption}{Assumption}{Assumptions}

\numberwithin{equation}{section}

\renewcommand{\epsilon}{\varepsilon}

\usepackage{babel}

\usepackage{xcolor}
\newcommand{\LATER}[1]{\error}
\newcommand{\fLATER}[1]{\error}
\newcommand{\TBD}[1]{\error}
\newcommand{\fTBD}[1]{}
\newcommand{\PROBLEM}[1]{\error}
\newcommand{\fPROBLEM}[1]{\error}

\title[]{Exchangeable modelling of relational data: checking sparsity, train-test splitting, and sparse exchangeable Poisson matrix factorization}
\newcommand{\mcl}{\mathcal}

\newcommand{\W}{\mcl{W}}

\newcommand{\defnphrase}[1]{\emph{#1}}

\newcommand{\Reals}{\mathbb{R}}

\newcommand{\NNReals}{\Reals_{+}}
\newcommand{\edges}{e}

\newcommand{\EE}{\mathbb{E}}
\renewcommand{\Pr}{\mathbb{P}}

\newcommand{\given}{\mid}

\newcommand{\poiDist}{\mathrm{Poi}}
\newcommand{\as}{\textrm{ a.s.}}

\newcommand{\intd}{\mathrm{d}}
\newcommand{\dist}{\ \sim\ }
\newcommand{\distiid}{\overset{\mathrm{iid}}{\dist}}
\newcommand{\distind}{\overset{ind}{\dist}}

\newcommand{\PP}{\Pi}
\newcommand{\PPDist}{\mathrm{PP}}

\global\long\def\iid{i.i.d.\ }

\global\long\def\bern{\mathrm{Bern}}

\global\long\def\bernDist{\bern}

\global\long\def\poiDist{\mathrm{Poi}}

\global\long\def\normDist{\mathrm{Normal}}

\global\long\def\gammaDist{\mathrm{Gamma}}

\global\long\def\multiDist{\mathrm{Multi}}

\global\long\def\tPoiDist{\mathrm{tPoi}}

\global\long\def\given{\mid}

\global\long\def\distiid{\overset{iid}{\dist}}

\global\long\def\distind{\overset{ind}{\dist}}

\global\long\def\Reals{\mathbb{R}}

\global\long\def\NNReals{\Reals_{+}}

\global\long\def\as{\textrm{ a.s.}}

\providecommand\given{} 
\newcommand\SetSymbol[1][]{
  \nonscript\,#1:\nonscript\,\mathopen{}\allowbreak}
\DeclarePairedDelimiterX\Set[1]{\lbrace}{\rbrace}%
{ \renewcommand\given{\SetSymbol[]} #1 }

\newcommand{\Ind}{\mathbbm{1}}

\newcommand{\ul}{u}
\newcommand{\Ul}{U}
\newcommand{\uft}{x}
\newcommand{\Uft}{X}
\newcommand{\il}{t}
\newcommand{\Il}{T}
\newcommand{\ift}{y}
\newcommand{\Ift}{Y}
\newcommand{\ufSpace}{\mathcal{X}}
\newcommand{\ufMeasure}{\eta}
\newcommand{\ifSpace}{\mathcal{Y}}
\newcommand{\ifMeasure}{\rho}
\newcommand{\usize}{s}
\newcommand{\isize}{\alpha}

\newcommand{\urm}{\mu}
\newcommand{\irm}{\rho}

\newcommand{\samp}[2][]{\mathsf{Smpl}\mathopen{}\left(#2,#1\right)\mathclose{}}

\acmDOI{11.11/111_1}

\acmISBN{111-1111-11-111/11/11}

\acmConference[WWW'2018]{The Web Conference}{2018}{Somewhere, World} 
\acmYear{2018}
\copyrightyear{2018}

\acmArticle{1}
\acmPrice{99.00}

\begin{document}

\author{Victor Veitch\and Ekansh Sharma \and Zacharie Naulet\and Daniel M. Roy}

\renewcommand{\Ind}{I}

\begin{abstract}
  A variety of machine learning tasks---e.g., matrix factorization, topic
  modelling, and feature allocation---can be viewed as learning the parameters of
  a probability distribution over bipartite graphs.  
  Recently, a new class of models for networks, the sparse exchangeable graphs, have been introduced
  to resolve some important pathologies of traditional approaches to statistical network modelling; 
  most notably, the inability to model sparsity (in the asymptotic sense).
  The present paper explains some practical insights arising from this work. 
  We first show how to check if sparsity is relevant for modelling a given (fixed size) dataset
  by using network subsampling to identify a simple signature of sparsity.
  We discuss the implications of the (sparse) exchangeable subsampling theory
  for test--train dataset splitting; we argue common approaches can
  lead to biased results, and we propose a principled alternative.  
  Finally, we study sparse exchangeable Poisson matrix factorization as a worked example. 
  In particular, we show how to adapt mean field variational inference to the
  sparse exchangeable setting, allowing us to scale inference to huge datasets.
\end{abstract}

\maketitle

\section{Introduction}
\label{sec:introduction}

Consider a dataset consisting of users, items, and links between users and items whenever a user has consumed a particular item.
A common task is to recommend unconsumed items to
users.
Loss functions that directly optimize the quality of the recommendation tend to be intractable,
which motivates learning a recommendation rule by instead minimizing some proxy.
A particularly natural choice is to fit a generative model for the dataset and base
recommendations on this model, e.g., by recommending items that have a high posterior probability of being ranked highly by the user.
Intuitively, the quality of the recommendation depends on the quality of the model.
The vast majority of relational data modelling approaches used in practice
fall under the remit of the \emph{dense exchangeable} (or \emph{dense graphon}) framework \cite{Aldous:1981,Hoover:1979,Orbanz:Roy:2015};
models in this class are intuitive and easy to work with, but have a limited
capacity to capture the structure of real-world data. 
In particular, they are misspecified as models for sparse data \cite{Orbanz:Roy:2015}.
Recently, the dense exchangeable framework has been generalized to accommodate a much larger 
range of phenomena. This new framework is known as the 
\emph{sparse exchangeable} (or \emph{graphex}) 
framework \cite{Caron:2012,Caron:Fox:2017,Herlau:Schmidt:Morup:2016,Veitch:Roy:2015,Borgs:Chayes:Cohn:Holden:2016,Veitch:Roy:2016,Janson:2016,Janson:2017,Todeschini:Caron:2016,Borgs:Chayes:Cohn:Veitch:2017,Caron:Rousseau:2017}.
By viewing relational datasets as bipartite graphs we are able to, in principle, 
apply the sparse exchangeable framework to generalize 
many existing approaches in machine learning and improve the quality of the underlying model.

This idea is widely applicable. 
For example, in topic modelling, we take one part of the bipartite graph to be words, the other part to be documents,
and the edges indicate whether each word is included in each document.
In feature allocation, one part of the graph is people, the other part is features, and the edges
indicate whether each feature is possessed by each person.
In recommendation,  
one part is users, the other part is items,
and each edge $(i,j)$ indicates that user $i$ has consumed item $j$.
For the sake of concreteness, we will use this user-item analogy throughout the remainder of the paper.

We call a sequence of growing bipartite graphs \defnphrase{dense} if a constant
fraction of all edges are present as either (or both) the number of users or
items increases, and we call a sequence \defnphrase{sparse} (or not dense) if it
is not dense.%

Folk wisdom holds that real-world relational datasets are sparse, and thus that the 
traditional modelling approach is unsatisfactory. However, sparsity is a
property of a sequence of observations, and in many data analysis situations we
just have a single observation at some particular size.  This begs the question:
does sparsity have practical relevance in this case?
To answer this question, we develop a simple
  method for directly checking whether sparsity is relevant for a given
  dataset. The key idea is to use the subsampling theory for sparse exchangeable
  graphs \cite{Veitch:Roy:2016, Borgs:Chayes:Cohn:Veitch:2017} to extract a
  sequence of shrinking subgraphs from the dataset; the sparsity of this
  sequence acts as a signature for whether sparsity needs to be accounted for in
  the modelling of the dataset.  We use this method to show that sparsity
  is indeed common in real-world datasets.

  Next, we treat the question of how a dataset should be split into testing and training sets. 
  This is a subtle question in the relational data setting, and often ad hoc methods are used.
  In the probabilistic approach a dataset $G_s$ is modeled as a realization of size $s$ from
  a probabilistic model with parameters $\Theta$, i.e., 
  \[
  G_s \given \Theta \dist \Pr(G_s \in \cdot \given \Theta).
  \]
  The key observation is that the test and train sets should be distributed as size $r$ and size $s-r$ samples
  from a distribution with common parameter $\Theta$; that is, the test and train sets should (marginally)
  come from the same distribution as the dataset. If this fails, then a model estimated using the training data
  will be biased for the test and full datasets.
  Accordingly, we must split the data using the unique subsampling scheme that has this property for the
  sparse exchangeable models.
  This leads to some complications that do not exist with more naive biased methods (e.g., holding out edges uniformly at random);
  we discuss how to deal with these in practice.
  Our test--train split scheme is mandated when using an exchangeable generative model (which covers most generative model cases),
  but is also likely more broadly useful because it gives a new approach assessing relational data models, 
  complementing existing approaches.

Finally, we treat sparse exchangeable Poisson matrix factorization as an extended worked example.
The generative model we use is essentially the bipartite version of the model of \cite{Todeschini:Caron:2016}.
The main novelty is a mean field variational inference scheme to scale inference to huge datasets.
Roughly speaking, our inference scheme adapts the Poisson matrix factorization approach of \cite{Gopalan:Hofman:Blei:2015},
with some new techniques to account for the additional complication introduced by the sparse structure.
Roughly, this generalizes Poisson matrix factorization to the sparse graph regime, although there are some
small differences between the dense version of our model and the approach of \cite{Gopalan:Hofman:Blei:2015}.
Our hope is that the techniques we introduce can serve as a general blueprint for scalable variational inference
of sparse exchangeable models. 
 
In practice, we implement inference using TensorFlow \cite{tensorflow2015-whitepaper};
this allows for GPU computation, resulting in very fast inference.\footnote{Code available at \href{https://github.com/ekanshs/graphex-nnmf}{github.com/ekanshs/graphex-nnmf}}
We find that accounting for sparsity carries negligible additional computational cost,
which is in contrast to the MCMC approach of \cite{Caron:Fox:2017,Herlau:Schmidt:Morup:2016,Todeschini:Caron:2016}.  
Interestingly, we find that accounting for sparsity has almost no effect on recommendation
performance, even for sparse datasets.
However, we do see an improvement in how well the model captures the structure of the data, e.g., the degree structure of the graph.

We define the sparse exchangeable models in \cref{sec:model}; 
this is a straightforward translation of the unipartite case.
Model checking is treated in \cref{sec:checking-sparsity},
and data splitting in \cref{sec:test-train-split}.
Finally, we cover sparse exchangeable Poisson matrix factorization in \cref{sec:gnnmf}.

\section{Bipartite Sparse Exchangeable Graphs}
\label{sec:model}

A bipartite graph $g$ is a triple $g = (V_U,V_I,E)$ where $V_U$, $V_I$
are sets of vertices and $E \subseteq V_U \times V_I$ is a set of (possibly weighted) edges.
We denote the number of users as $|V_U| = U(G)$, the number of items as $|V_I| = I(G)$, and,
abusing notation, the number of edges as $|E| = E(G)$.

\subsection*{The model}

The model for bipartite graphs we consider in this paper is the natural
extension of the sparse exchangeable unipartite graphs of
\cite{Caron:Fox:2017,Veitch:Roy:2015,Borgs:Chayes:Cohn:Holden:2016}, also
considered in \cite{Caron:2012,Caron:Fox:2017}. %
The basic structure is that each user $\Ul_i$ is assigned some latent features $\Uft_i$, each item $\Il_j$ 
is assigned some latent features $\Ift_j$, and, given the latent features,
and each edge $g_{ij}$ is drawn independently from some distribution parameterized by $W(\Uft_i, \Ift_j)$
for some function $W$ (the \emph{graphon}). 
For example, in rank $K$ matrix factorization the features are $K$-dimensional vectors
and typically $g_{ij} \given \Uft_i, \Ift_j \dist \normDist(W(\Uft_i, \Ift_j), 1)$ with $W$ the inner product function.
For concreteness, we will default to the case $g_{ij}\in \{0,1\}$---understood as an indicator for edge inclusion---and $W(\Uft_i, \Ift_j) = \Pr(g_{ij} = 1 \given \Uft_i, \Ift_j)$ 
for the rest of the paper, but all of our discussion applies to the more general setting.

The generative structure we have described thus far is common to both the dense exchangeable and sparse exchangeable models.
The distinction lies in how the latent features are generated.
In dense exchangeable models, these are drawn \iid from some probability space. 
This is a seemingly natural choice, but it is the root cause of the pathological denseness of the models.
In the sparse exchangeable model,
the user feature space, $(\ufSpace,\ufMeasure)$, and item feature space, $(\ifSpace,\ifMeasure)$, 
are taken to be infinite measure spaces.
The users are then drawn as a Poisson process $\PP_U$ with mean measure $\ufMeasure(\intd \Uft) \intd \Ul$ on $\ufSpace \times \NNReals$,
where the $\NNReals$ coordinates are interpreted as labels of the users; see \cref{fig:expl_diagram}.
The items $\PP_I$ are drawn in the analogous way. 
Edges are then randomly generated between users and items according to
\[
e_{ij} \given \PP_U,\PP_I \distind \bernDist(W(\Uft_i, \Ift_j)).
\]

This samples in an infinite bipartite graph. To restrict to finite size,
we include only users with labels $\Ul_i < \usize$ and items with labels $\Il_j < \isize$;
see \cref{fig:expl_diagram}. We refer to $\usize$ as the user-size of the graph,
and $\isize$ as the item-size; these naturally correspond to the sample size of a dataset.
This restriction results in an induced subgraph with a finite number of edges,
but with an infinite number of items and users. 
This is resolved by excluding any point of the Poisson processes that fails to connect to any edge;
that is, this model excludes users and items with degree zero. 

Summarizing, sparse exchangeable graph distributions are naturally parameterized by $\ufMeasure, \ifMeasure, W$,
and the notion of sample size is given by $\usize$ and $\isize$.
The analogous statement for the (more familiar) dense case is that dense exchangeable graph distributions are
naturally parameterized by $P, Q, W$---where the user features are drawn \iid according to distribution $P$ 
and the item features according to $Q$---and the notion of sample size is the number of users and items in the graph.

A few remarks are in order:

We defined the sparse exchangeable models on infinite measure spaces. The same definition works on finite measure spaces,
in which case the models are the dense exchangeable (\iid features) models, up to some minor technical differences.

We emphasize again that the distinction between the dense and sparse models is simply how the latent features are generated;
for this reason, it is relatively easy to postulate sparse analogues of machine learning models that are already used in practice.

The models used here may seem somewhat arbitrary. 
This is not so; these models are the natural extension of the dense exchangeable theory,
and are derived from simple and natural postulates. See \cite{Veitch:Roy:2015, Borgs:Chayes:Cohn:Holden:2016}
for a derivation from exchangeability, and \cite{Borgs:Chayes:Cohn:Veitch:2017}
for a derivation from network subsampling invariance.

\begin{figure}[!htb]
  \centering
  \includegraphics[width=0.9\linewidth]{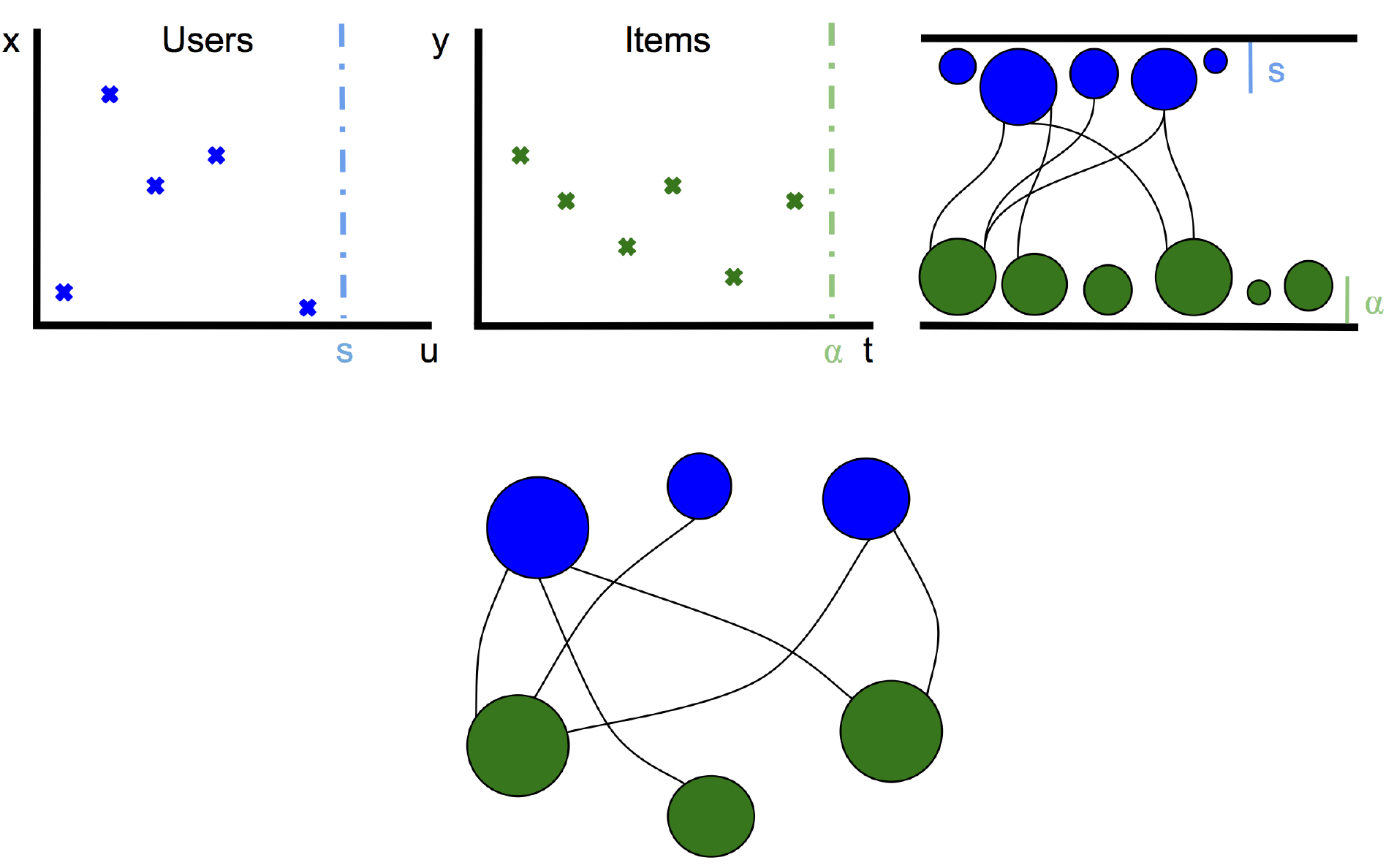}
  \caption{Typical realization of a graphex process. The left panel represents
    a sample from a Poisson process on $\NNReals^2$ restricted to
    $[0,\usize]\times \NNReals$. The x-axis corresponds to users labels, while
    the y-axis corresponds to latent features. The middle panel shows a typical
    sample of the item Poisson process. The right panel displays a sample of
    edges between users and items, sampled independently according to
    $W(\uft_i,\ift_j)$.  The bottom panel displays the sampled graph, given by
    deleting the atoms of each Poisson process that do not connect to any
    edges.}
  \label{fig:expl_diagram}
\end{figure}

\subsection*{Subsampling}
\label{sec:sampling}
We make extensive use of the following scheme for sampling random subgraphs from a graph \cite{Veitch:Roy:2016}:
\begin{definition}
  Let $p,q \in [0,1]$. A \defnphrase{$(p,q)$-sampling} $\samp[(p,q)]{g}$ of a
  bipartite graph $g$ is a random subgraph of $g$ given by including each user
  of $g$ independently with probability $p$ and each item of $g$ independently
  with probability $q$, and returning the induced subgraph with the isolated
  vertices removed.
\end{definition}
This is the sampling scheme associated with sparse exchangeable graphs: If
$G_{\usize,\isize}$ is generated according to $\ufMeasure, \ifMeasure, \W$ then
the subgraph $\samp[(p,q)]{G_{\usize,\isize}}$ is equal in distribution to
$G_{p\usize,q\isize}$ generated according to $\ufMeasure, \ifMeasure, \W$
\cite{Veitch:Roy:2016}.  That is, this sampling scheme defines the relationship
between the generated graphs at different sizes. In fact, this is a defining property of the (sparse) exchangeable graphs, 
and (sparse) exchangeability can be understood as equivalent to this invariance.\cite{Borgs:Chayes:Cohn:Veitch:2017}

\section{Checking Sparsity}
\label{sec:checking-sparsity}
Sparsity is a property of a sequence of graphs, but typically analysis is performed on some particular, fixed size, dataset.
Accordingly, it may seem that the ability to model sparsity is not relevant for most applications in practice. 
We show that this intuition is incorrect: sparsely generated datasets have a readily identifiable signature.
Accordingly, we can assess ahead of time whether sparse structure is present in our data, and thus
whether we should incorporate sparsity into the model.  

Define the edge density of a bipartite graph $g$ to be $\rho(g) = \edges(g) / (U(g) I(g))$.
In the dense case, $\rho(G_{\usize,\isize})$ is constant with respect to $\usize, \isize$, but in the sparse
case it decreases as either (or both) $\usize$ or $\isize$ increases. 
Thus, if we could observe the graph process at different values of $\usize,\isize$, 
we could observe sparsity as a change in the value of the edge density. 
The key insight is that $(p,q)$-sampling allows us to effectively simulate this;
intuitively, this is because, marginally, a $(p,q)$-sampling is a sample of the
graph at user size $p\usize$ and item size $q\isize$.

The content of the following theorem is that this intuition carries through even conditional on the observed data.
\begin{theorem}
If $(G_{\usize,\isize})$ is dense then, for $p,q > 0$,
\[
\lim_{\usize,\isize\to\infty}\frac {\rho(\samp[(p,q)]{G_{\usize,\isize}})}  {\rho(G_{\usize,\isize})} = 1 \as 
\]
\end{theorem}
\begin{proof}
  (Sketch) From the definition of dense, $\rho(G_{\usize,\isize}) \to c$ almost
  surely as $\min(\usize,\isize) \to \infty$, for some constant $c$. The result
  follows because the distributional invariance of sparse exchangeable models
  under $(p,q)$-sampling \cite{Veitch:Roy:2016,Borgs:Chayes:Cohn:Veitch:2017}
  guarantees that $\rho(\samp[(p,q)]{G_{\usize,\isize}}) \rightarrow c$, with
  the same the same limit $c$.
\end{proof}

Accordingly, we can check if a graph is sparsely generated by plotting 
the edge density of subsampled graphs against the $(p,q)$-sampling levels;
see \cref{fig:sparsity-test}.  
This is theoretically sound if the model is generated according to an exchangeable model,
which is a common assumption in generative modelling of relational data.
Otherwise, this strategy simply provides a powerful heuristic for assessing sparsity. 

\begin{figure}[!t]
  \centering
  \subfloat[Dense users, Dense items]{
  \includegraphics[width=0.5\linewidth]{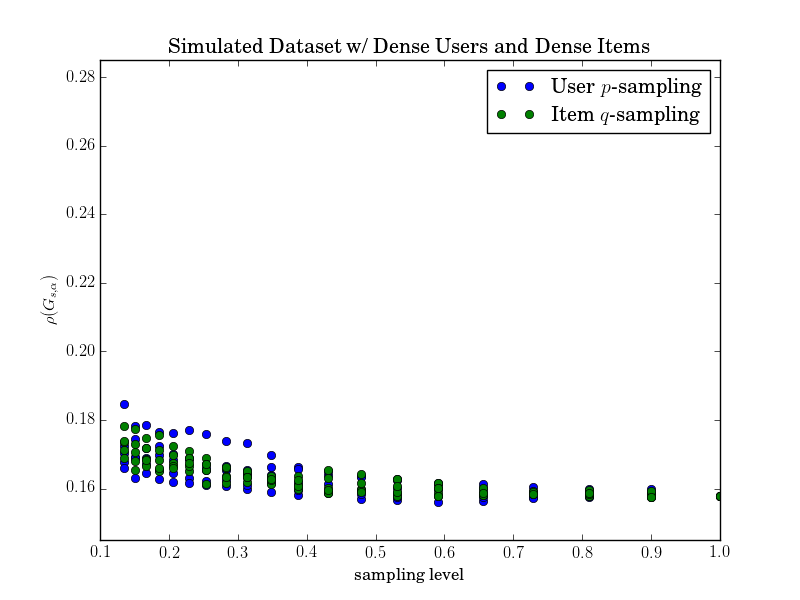}
  \label{fig:sparsity-test-sim-ss}
  } 
  \subfloat[Dense users, Sparse items]{
  \includegraphics[width=0.5\linewidth]{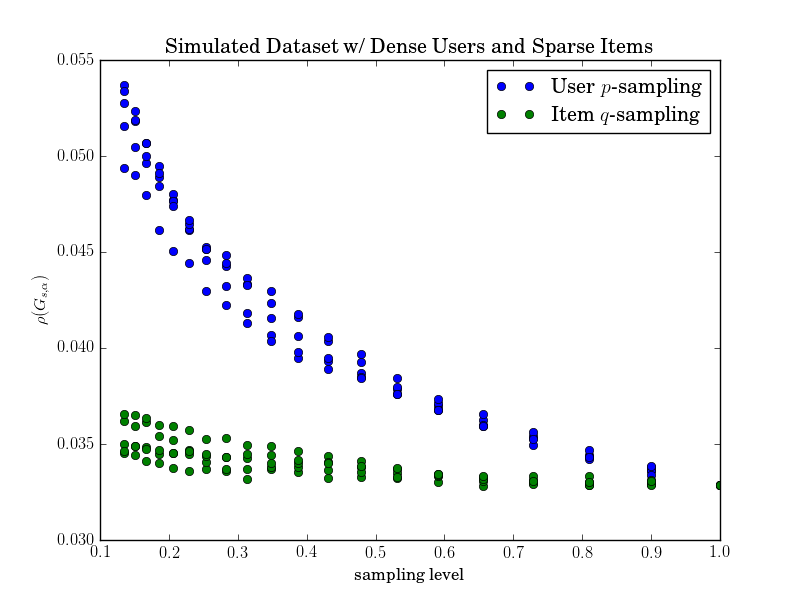}
  \label{fig:sparsity-test-sim-ds}
  } 
    
  \subfloat[Real-World Datasets]{
  \includegraphics[width=0.5\linewidth]{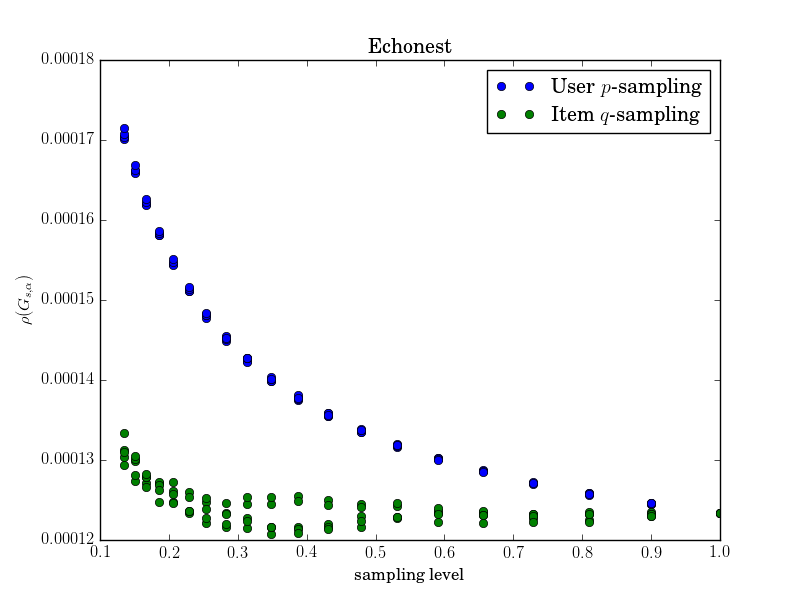}
  \includegraphics[width=0.5\linewidth]{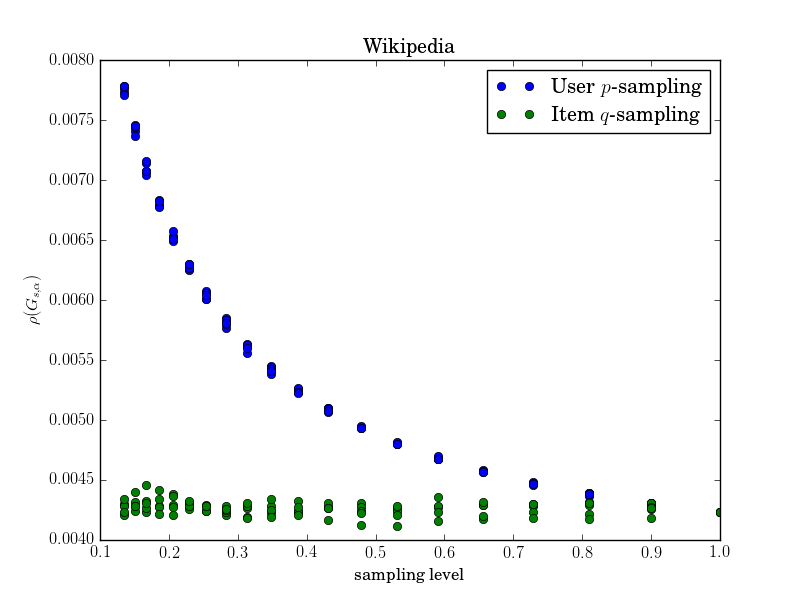}
  \label{fig:sparsity-test-rwd}
  }
  \caption{Checking sparsity:
    Each plot shows the edge density of randomly sampled subgraphs plotted against the sampling level.
    More precisely, for each graph $g$ we produce random subgraphs $\samp[(p,1)]{g}$ by $p$-sampling users,
    and we plot the edge density of these graphs against $p$. We follow the same procedure for $q$-sampling of items.
    If the edge density is approximately constant with $p$ (respectively $q$) then we consider the graph to be dense.
    (a) the graph is a simulated dataset where user nodes and items 
    nodes are generated from finite activation Poisson process---users and items are dense. 
    (b) the graph is a simulated dataset where user nodes are generated from a finite activation process but item nodes are generated from infinite activation Poisson process---intuitively, items are sparse and users are dense. 
    (c) Graph densities for the Wikipedia and Echonest datasets described in \cref{datsets}. 
    The plots show that item nodes should be modelled with an infinite activation process; that is, sparsity is present in this data.}
  \label{fig:sparsity-test}
\end{figure}

\section{Test--Train Split}
\label{sec:test-train-split}
Model evaluation is a key component of data analysis. 
Often, this involves randomly splitting the available data into a test set and a training set.
There are many seemingly natural ways to partition a bipartite graph,
so some care is required in splitting the data.
The choice of partitioning scheme may induce a significant sampling bias
in the test and training sets,
impeding evaluation and complicating model comparison.
In this section we give an approach motivated by the sampling theory of sparse exchangeable graphs.
This approach is the 'right' one for exchangeable models (including dense ones), 
and is also useful as complement to existing ad hoc approaches for evaluating non-generative models.

To see the difficulty with test--train splitting in the relational data setting,
consider a common evaluation procedure for recommender systems:
A test set is produced by holding out reviews independently at random.
For each user with at least one heldout ranked item, the trained model is asked to recommend items by
ranking the set of all films the user has not rated in the training data (i.e., the non-ratings and the heldout data).
Then performance is scored by how highly the algorithm ranks the heldout data.
Notice that this test--train split induces a degree biased sampling of the users and items; that is,
we expect that a typical user in the test set consumes more items
than a typical user in the training set, and that a typical item
in the test set is more popular than a typical item in the training set.
This means that evaluation procedures based on such a split 
tend to focus on the task of recommending popular items to 
popular users---this is often not a good proxy for the true task of interest.
It is easy to modify the scheme to address this particular problem, 
but in general it remains a concern that any test--train scheme may induce some more subtle sampling bias.

In the generative model setting, we must choose the splitting procedure such that 
the parameters of the generative model can be consistently estimated from the training set.
For (sparse) exchangeable models, this means partitioning via $(p,q)$-sampling.
Concretely, we propose the following data splitting scheme; see \cref{fig:tt-split}.
\begin{enumerate}
\item Sample $G_\mathrm{train} = \samp[(p,1)]{G}$, and take $G_\mathrm{holdout}$ to be the complement of $G_\mathrm{train}$ in $G$
\item Sample $G_\mathrm{test} = \samp[(1,q)]{G_\mathrm{holdout}}$ and take $G_\mathrm{holdoutfit}$ to be the complement of $G_\mathrm{test}$ in $G_\mathrm{holdout}$
\end{enumerate}
The point of the scheme is that if $G$ is drawn as a $\isize$ item-size, $\usize$ user-size
sparse exchangeable graph generated according to $\ufMeasure, \ifMeasure, \W$,
then $G_\mathrm{train}$ and $G_\mathrm{test}$ are distributed as, respectively,
size $p\usize, \isize$ and size $p\usize,q\isize$ graphs generated according to
$\ufMeasure, \ifMeasure, \W$.  If we had used some other sampling scheme then,
generally, the test and train sets would be distributed according to different
distributions, and a model that performed well on the training set would not be
expected to perform well on the test set or on fresh data.
\begin{figure}[t]
  \centering
  \includegraphics[width=0.9\linewidth]{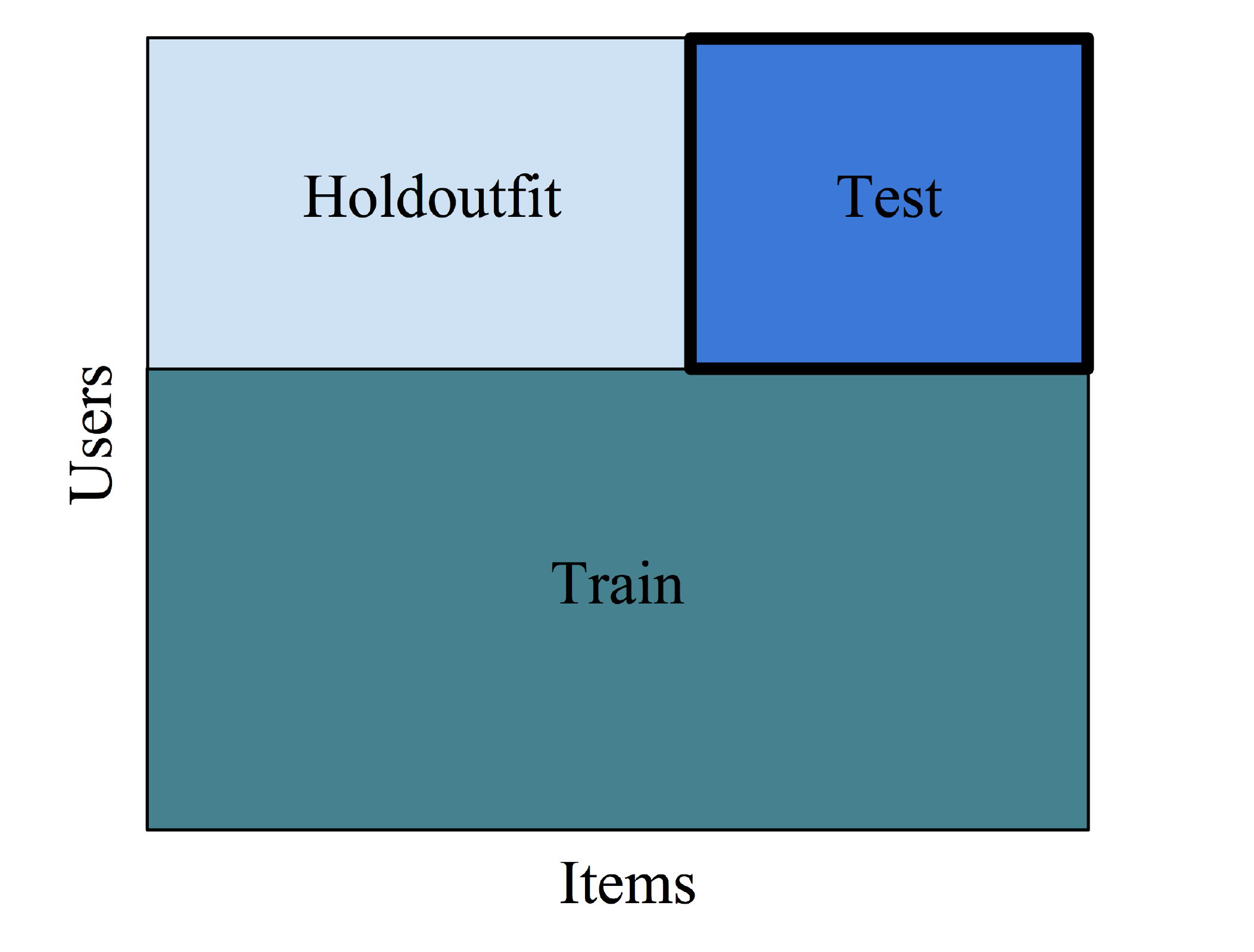}
  \caption{Illustration of the outcome of a test--train split procedure. The dataset is first
    split into a train set and a holdout set by subsampling the users part of
    the graph. Then, the holdout set is in turn split into a test set and a
    holdoutfit set by subsampling the items part of the graph.
    We use $(p,q)$-sampling, which means that any rows or columns that are empty (that is, all $0$)
    in any of the partition sets are removed from that set.
}
  \label{fig:tt-split}
\end{figure}

For concreteness, we envision an estimation procedure as an algorithm that takes an observed bipartite graph $G$ and
outputs estimates $\hat{\ufMeasure}, \hat{\ifMeasure}, \hat{W}$ for the model parameters, and $\hat{\uft}_i$ for the latent feature of each user, 
and $\hat{\ift_j}$ for the latent feature of each item. 
The output of the estimation procedure on $G_\mathrm{train}$ includes estimates for $\hat{W}$
and for the latent features of the items in $G_\mathrm{test}$.
However, it does not include estimates for the latent features of the users in the test set.
Indeed, $G_\mathrm{train}$ does not carry information about these values.
This motivates splitting the holdout set: the validation procedure
should use $\hat{W}$, $\hat{\ufMeasure}$, and $G_\mathrm{holdoutfit}$ to produce an estimate for the 
latent user features of the test set.
This last step is possible for exchangeable random graph models because of the conditional independence structure:
to estimate the latent feature of any user, it is sufficient to know $W$, $\ufMeasure$, and the latent features of its neighbours.

A concrete example given in \cref{VI_data_split}.
\cref{fig:ppc} shows degree distribution samples from the approximate posterior distribution over the test sets
from a model trained using the training sets, under our proposed test--train splittings scheme.
This figure also plots the true degree distribution from the test--data.
Note that the model fit to the training data is able to accurately predict the structure of the test data.

\section{Sparse Poisson Matrix Factorization}\label{sec:gnnmf}
We now turn to non-negative matrix factorization as an extended example.
We adapt the model of \cite{Todeschini:Caron:2016},
using the bipartite variant as a component of probabilistic matrix factorization.
We emphasize that the model we present here is chosen for its simplicity, 
with the aim of exposing sparsity considerations as clearly as possible.
The aim is not to write down the best possible model for the recommendation task we consider,
and indeed there are several obvious extensions that we omit because 
they are essentially orthogonal to the study of sparsity. 

\subsection{The generative model}
\label{sec:generative-model}
Let $\PPDist(m)$ denote the distribution of a Poisson process with mean measure $m$,
and $\gammaDist(a,b)$ denote a gamma distribution with shape $a$ and rate $b$.

\subsubsection{Generalized Gamma Process}
We follow the approach of \cite{Caron:2012} and its successors
\cite{Caron:Fox:2017, Herlau:Schmidt:Morup:2016, Todeschini:Caron:2016} based on
the \textit{Generalized Gamma Process} (GGP),
a Poisson process with mean measure
\[
g_{\sigma,\tau}(\intd u) = \frac{1}{\Gamma(1-\sigma)} u^{-(1+\sigma)}\exp(-\tau
u)\,\intd u,
\]
for $(\sigma,\tau) \in (0,1)\times (0,\infty)\cup (-\infty,0]\times [0,\infty)$.\footnote{Related
Bayesian non-parametric tools have previously been used in matrix
factorization to allow for infinite latent factor dimension, e.g., \cite{Hoffman:Blei:Cook:2010};
we emphasize that this is substantively unrelated to what we are doing here.}

For our purposes, there are two important facts about the GGP.
First, points of the GGP are equivalent to \iid draws from a Gamma distribution if and only if $\sigma < 0$.
In the model defined below this means the model is dense if and only if $\sigma_U,\sigma_I < 0$ \cite{Caron:Fox:2017}. 
This property allows for easy, interpretable comparison between sparse and dense variants of the model.%

Second, the GGP admits a pseudo-conjugacy relationship with the Poisson
distribution. Informally, if
$\xi = \Set{(\theta_1, \omega_1),(\theta_2,\omega_2),\dots}$
is a GGP with parameters $\sigma,\tau$
and
$N_i \given \xi \dist \poiDist(\lambda \omega_i)$ then $\xi \given N$ is equal
in distribution to a point process with two independent components: One component consists
of all atoms for which $N_i = 0$, this is distributed as another GGP with parameters $\sigma$ and $\tau + \lambda$. The other component
consists of the atoms for which $N_i > 0$; in this case the posterior masses are
distributed as $\gammaDist(N_i-\sigma,\lambda+\tau)$, independent of each other
and of the first part of the process. Intuitively speaking, this is the
posterior distribution $\Pr(\omega_i \in \cdot \given N_i)$ of the observed weights.  In the
case that $\sigma < 0$, this is the same conjugate posterior update that we
would arrive at by taking $\omega_i \distiid \gammaDist(-\sigma,\tau)$; hence
the pseudo-conjugacy terminology. It is this property that makes the GGP amenable to
efficient inference.

\subsubsection{Sparse Poisson Matrix Model}
\label{sec:generative-model-1}

The model assigns each item and user a $K$ dimensional latent feature.  
The basic structure is:
\begin{enumerate}
  \item Each user $i$ is assigned a total popularity $P_i$ and affinities
  $(\phi_{ik})_{k=1}^K$ to each feature dimension $k$.
  \item Each item $j$ is assigned a total popularity $P'_j$ and affinities
  $(\varphi_{jk})_{k=1}^K$ to each feature dimension $k$.
  \item Given the features, each edge $(i,j)$ is included independently with probability
  $1-\exp(P_iP'_j\sum_k\phi_{ik}\varphi_{jk}).$
\end{enumerate}
The inclusion probability is the probability of a non-zero draw from a Poisson
distribution with mean $P_iP'_j\sum_k\phi_{ik}\varphi_{jk}$; this is done to
allow us to exploit the GGP-Poisson pseudo-conjugacy for inference.  

We present the generative model in a somewhat different form to allow for easy
derivation of the inferential updates.
The user parameters are:
\begin{gather*}
  \xi = \Set{(\gamma_1,\ul_1),(\gamma_2,\ul_2),\dots} \dist
  \PPDist(g_{\sigma_U,\tau_U} \times \lambda),\\
  \theta_{ik} \distiid \gammaDist(a,b),
\end{gather*}
The popularity of user atom $i$ is $P_i = \gamma_i \sum_k \theta_{ik}$, and the
affiliations are $\phi_{ik} = \theta_{ik} / \sum_k\theta_{ik}$. The item
parameters are:
\begin{gather*}
  \zeta = \Set{(\omega_1,\il_1),(\omega_2,\il_2),\dots} \dist
  \PPDist(g_{\sigma_I,\tau_I} \times \lambda),\\
  \beta_{jk} \distiid \gammaDist(c,d).
\end{gather*}

Let $g_{ij}=1$ indicate that edge $(\ul_i,\il_j)$ is included in the graph (and
$g_{ij}=0$ otherwise).  The generative model for the connections is:
\begin{gather*}
  e_{ij}^k \given \gamma_i, \{\theta_{ik}\}_{k=1..K}, \omega_j, \{\beta_{jk}\}_{k=1..K} \distind \poiDist(\gamma_i \omega_j \theta_{ik} \beta_{jk})\\
  g_{ij} = \Ind[\textstyle\sum_k e_{ij}^k > 0].
\end{gather*}
The use of $K$ latent edge counts per user-item pair should be viewed as an auxiliary variable technique;
this idea is borrowed from \cite{Gopalan:Hofman:Blei:2015}. Taking $g_{ij}=\sum_k e_{ij}^k$ instead results in a model
for integer valued relations; the rest of our discussion holds for this case, subject to obvious minor modifications.

Note that the generative model is for an infinite graph. As usual, we restrict to finite graphs by truncating the label spaces of the users and items,
and then discarding any users or items that are isolated in the induced subgraph; i.e., we restrict to user atoms for which $\ul_i < \usize$
and item atoms such that $\il_j < \isize$.
Recast in the general language used earlier in the paper: 
\begin{gather*}
  W( (\gamma, \boldsymbol{\theta}), (\omega,\boldsymbol{\beta}) ) = 1 -
  \exp\left(- \gamma \omega \textstyle\sum_{k=1}^K \theta_k\beta_k \right),\\
  \ufMeasure(\intd \gamma \intd \boldsymbol{\theta}) =
  g_{\sigma_U,\tau_U}(\intd \gamma) \prod_{k=1}^K
  \frac{\theta_k^{a-1}\exp(-b\theta_k)\intd \theta_k}{b^{-a}\Gamma(a)},\\
  \ifMeasure(\intd \omega\intd \boldsymbol{\beta}) =
  g_{\sigma_I,\tau_I}(\intd \omega) \prod_{k=1}^K
  \frac{\beta_k^{c-1}\exp(-d\beta_k)\intd \beta_k}{d^{-c}\Gamma(c)}.
\end{gather*}

\subsection{Inference}
The learning task has two components: inferring the parameters of $\ufMeasure$ and $\ifMeasure$ (noting that $W$ is fixed),
and inferring the latent feature values of the users and items.

\subsubsection{Global parameters}
\label{sec:hyperparameters}
The sparse model requires us to set hyperparameters that are not relevant in the
dense setting: namely, the GGP parameters $\sigma_U$ and $\sigma_I$, and the
sizes $\usize$ and $\isize$.  We use the general estimation strategy of \cite{Naulet:Sharma:Veitch:Roy:2017}.  This provides a consistent estimator $\hat{\sigma}_U$
for $\sigma_U$ as a function of the degrees of the users in the dataset.
Namely, for a graph $g$ where each user $i$ has degree $d_i$,
\begin{equation*}
  \hat{\sigma}_U =
  \frac{\log(U(g)) - \log(2^{-1} \textstyle\sum_i(1- 2^{-d_i}))}{\log 2}.
\end{equation*}
The estimator for $\sigma_I$ is the obvious analogue. 

It can also be shown that, for some slowly growing function $C_{\isize}$ of $\isize$ (e.g., $C_\isize = O(\log\log\isize)$), 
\[
\log U(G_{\usize,\isize}) - \sigma_U \log e(G_{\usize,\isize}) = C_{\isize} + (1-\sigma_U) \log(s) + o_s(1),
\]
see \cite{Caron:Rousseau:2017} for a derivation of the asymptotics in the (harder) unipartite case; 
the bipartite case is a straightforward adaption.
We can estimate $C_\alpha$, which depends on the model, by treating it as a constant with respect to $\alpha$, 
and estimating this constant by simulating data (with known user size) from the model. 
We then set the user size by subbing $\hat{\sigma}_U$ for $\sigma_U$, throwing away the $o_s(1)$ term, and solving for $\usize$.
We also use the analogous strategy for setting $\isize$.
 
We use an empirical Bayes type procedure, fixing these hyperparameters to their estimated values,
and using a more sophisticated approach to estimating the posterior distribution of the latent features.

\subsubsection{Latent features}

We adopt a mean field variational inference (VI) approach, approximating the true posterior distribution over the
parameters by a fully factorized distribution. See \cite{Blei:Kucukelbir:McAuliffe:2016} for a review of variational inference.
Nominally, main challenge here is that the likelihood has no closed form expression---it can't be readily differentiated, 
or even evaluated---and it is computationally expensive to draw samples from. 
This disallows most general purpose variational inference algorithms.

The solution is that we have parameterized the model such that the complete conditional distribution of each variable
is (approximately) an exponential family distribution---we demonstrate this below, and discuss the required approximations (\cref{complete_conditionals}). 
Complete conditional distributions are the conditional distributions given the data and all other variables. 
The significance of this property is that it allows us to read off a \textit{Coordinate Ascent Variational Inference} scheme (CAVI) automatically \cite{Ghahramani:Beal:2001, Hoffman:Blei:Wang:Paisley:2013}.
It is rather remarkable that this works: CAVI was developed for conjugate Bayesian models with \iid observations, but in the sparse graph case
there is no possible independent prior on the user weights or item weights that would reproduce the model.
Nevertheless, the pseudo-conjugacy of the GGP suffices for efficient inference.

Broadly, the structure of the resulting algorithm is that each variable (e.g., $\theta_{ik}$) gets a parameterized approximating distribution (e.g., $q(\theta_{ik} \in \cdot \given \phi_{ik})$), and the algorithm learns the parameters by iteratively updating them given all other parameters.
Complete conditionals in exponential family form allow us to read off the form of $q$ and the parameter updates.  

\subsubsection{Complete Conditionals}
\label{complete_conditionals}
Most of the results presented here are readily derived 
by ordinary conjugate update manipulations in combination with the generalized gamma process--Poisson pseudo-conjugacy; 
additionally, a detailed treatment of the unipartite case is given in \cite{Todeschini:Caron:2016}.

Let $e_{i\cdot}^k = \sum_j e_{ij}^k$ and $e_{i\cdot} = \sum_k e_{i\cdot}^k$, and 
let $\urm^{\usize}_k = \sum_{i:e_{i\cdot}=0} \gamma_{i}\theta_{ik} 1[\ul_i < \usize]$. 
Intuitively speaking, $\urm^{\usize}_k$ is the total mass of type $k$ belonging to user atoms that
have failed to connect to any items; this turns out to be a sufficient statistic for user atoms that do not connect to any items.
We also need the analogous definitions for the items, writing $\irm^{\isize}_k$ for the total mass of type $k$
belonging to item atoms that have failed to connect to any users. Then,
\begin{align*}
  \theta_{ik}\given \xi,\zeta,\beta,e &\dist \gammaDist(a+
  e_{i\cdot}^{k},b+\gamma_{i}(\textstyle\sum_{\Set{j:e_{\cdot
        j}>0}}\omega_{j}\beta_{jk}+\irm^{\isize}_k)),\\
  \beta_{jk}\given\xi,\zeta,\theta,e & \dist \gammaDist(c+e_{\cdot j}^{k}, d +
  \omega_{j}(\textstyle\sum_{\Set{i:e_{i\cdot}>0}} \gamma_{i}
  \theta_{ik}+\urm^{\usize}_k)).
\end{align*}

The GGP weights for atoms that connect to at least 1 edge:
\begin{align*}
  &\gamma_{i} \given \zeta,\theta,\beta,e\\
  &\qquad \dist \gammaDist( -\sigma_{U}+ e_{i\cdot},
    \tau_{U}+\textstyle\sum_{k}\theta_{ik}(\sum_{\Set{j:e_{\cdot
    j}>0}}\omega_{j}\beta_{jk}+\irm^{\isize}_k)),\\
  &\omega_{j} \given \xi,\theta,\beta,e\\
  &\qquad\dist \gammaDist( -\sigma_{I} + e_{\cdot j}, \tau_{I} +
    \textstyle\sum_{k} \beta_{jk}(\textstyle\sum_{\Set{i:e_{i\cdot}>0}}
    \gamma_{i} \theta_{ik}+\urm^{\usize}_k)).
\end{align*}
Notice that in the case $\sigma_U < 0$, the complete conditional for $\gamma_i$ corresponds to a $\gammaDist(-\sigma_U, \tau_U)$ prior distribution. However, in the case that $\sigma_U > 0$, there is \emph{no} independent prior on $\gamma_i$ that could have given rise to this posterior.  

Next, the auxiliary variables:
\begin{multline*}
  (e_{ij}^{1},\dots,e_{ij}^{K})\given\theta,\beta, \xi, \zeta, [g_{ij}=1]\\
  \dist\tPoiDist(\gamma_{i}\omega_{j}\theta_{i1}\beta_{j1},\dots,\gamma_{i}\omega_{j}\theta_{iK}\beta_{jK}),
\end{multline*}
where $\tPoiDist(\lambda_1,\dots,\lambda_K)$ is the $K$-dimensional truncated Poisson distribution defined by the following scheme. For drawing a sample
$(N_1, \dots, N_K)$: Draw $N \dist \poiDist(\sum_k \lambda_k)$ conditional on
$N>0$, and then draw $(N_1,\dots,N_k) \given N \dist \multiDist(\lambda_1 /
\sum_k \lambda_k, \dots, \lambda_K / \sum_k \lambda_k)$, see \cite{Todeschini:Caron:2016}.

The complete conditional of $(\urm^{\usize}_1, \dots, \urm^{\usize}_K)$ is not available in closed form.
However, using point process techniques, we can derive the 
conditional expectation and variance, see \cref{sec:leftover_cond}.
We find that $\mathrm{var}[\urm^{\usize}_l \given \zeta,\beta,e] \ll \EE[\urm^{\usize}_l \given \zeta,\beta,e]$ in the large data regime. 
This motivates the approximation $\Pr(\urm^{\usize}_k \in \cdot  \given \zeta,\beta,e) \approx \delta_{\EE[\urm^{\usize}_l \given \zeta,\beta,e]}$.
That is, we simply ignore the variability. This approximation works well in practice when used as an input to our inference procedure.

The expectation of the complete conditional is given by:
\begin{equation*}
  \EE[\urm^{\usize}_l \given \zeta,\beta,e] =
  s\EE\left[%
    \frac{\theta^*_l}{\tau_U + \sum_k \theta^*_k \sum_j \omega_j \beta_{jk}}
    \given \zeta,\beta,e
  \right]
\end{equation*}
for $\theta^*_l \distiid \gammaDist(a,b)$; this is easy to approximate with Monte Carlo sampling.

We also use the analogous approximation for the complete conditional of $(\irm^{\usize}_1, \dots, \irm^{\usize}_K)$.

\subsection{Empirical study}
This section covers an empirical comparison of dense and sparse exchangeable Poisson matrix factorization. 
The main takeaways are:
\begin{enumerate}
\item The inference algorithm works well. Our algorithm is able to accurately recover the structure of simulated data, suggesting that the various approximations involved in our inference scheme are valid in practice.
\item The sparse model does a better job recovering the graph structure of sparse data, including real world data.
\item However, despite this, there is no appreciable difference in recommendation performance between the sparse and dense models. 
\end{enumerate}

\subsubsection{Datasets}
\label{datsets}
We consider the following datasets:
\begin{enumerate}
  \item \textbf{Simulated Dataset}: A sample from the model generated according to the following parameters: 
  $\sigma_U=\sigma_I=0.2$, $\tau_U=\tau_I=1.0$,  $a=b=c=d=0.1$, $\usize=\isize=1200.0$. 
  Dataset consists of 9.7M edges with 40,565 users and 40,768 items.
  \item \textbf{Netflix}: Consists of ratings that users have assigned to movies. In our experiment, we treat this as a simple graph by including an edge in the graph whenever user has rated a movie.
  The resulting dataset consists of 480,189 users, 17,770 movies (items), and 100M ratings (edges).
  \item \textbf{Echonest}: The Echonest taste profile dataset \cite{Bertin-Mahieux:Ellis:Whitman:Lamere:2011} is a music dataset with 1,019,318 users, 384,546 songs (items) and 48M entries where each entry is the number of times a user played a song. We treat this as a simple graph by including an edge whenever user played a song. 
  \item \textbf{Wikipedia}: The document-term matrix corresponding to WikiText-103 dataset \cite{Merity:Xiong:Bradbury:Socher:2016},
    after removing stop words, very common words, and stemming the word set.
    The dataset contains 29,425 documents (users), 161,085 terms (items), and 21M tokens (edges).
\end{enumerate}

\subsubsection{Hyperparameters}
For the sparse models, we set the sparsity parameters $\sigma_U$ and $\sigma_I$, and the sizes $\usize$ and $\isize$, 
according to the estimation scheme described above.
For dense models, we set $\sigma_U=\sigma_I=-0.1$ (and $\usize=\isize=0$). 
Empirically, we find that the performance of the dense model is robust to the choice of $\sigma_U$ and $\sigma_I$,
the analogous observation was also made in \cite{Gopalan:Hofman:Blei:2015}.

For all experiments, we set $\tau_I = \tau_I = 1.0$, $a=c=b=d=0.1$ and $K=30$. 
Model performance is somewhat dependent on these parameter choices, 
but our conclusions seem to hold generally. 

\subsubsection{Data splitting and posterior predictive model}
\label{VI_data_split}

We partition each dataset $G$ into $G_\mathrm{train}, G_\mathrm{test}$, and $G_\mathrm{holdoutfit}$ 
according to the scheme describe in \cref{sec:test-train-split}, taking $p=q=0.2$.
We now specialize the discussion of \cref{sec:test-train-split} to the sparse exchangeable Poisson matrix factorization model.

To evaluate the quality of our model, 
we would like to find the posterior predictive distribution
of the test set given $G_\mathrm{train}$ and $G_\mathrm{holdoutfit}$;
i.e., $\Pr(G_\mathrm{test} \in \cdot \given G_\mathrm{train} , G_\mathrm{holdoutfit})$.
We could then, e.g., recommend items to users in the test set by
recommending those items such that $\Pr(e_{ij}=1 \given G_\mathrm{train} , G_\mathrm{holdoutfit})$ is large.

Letting $\Theta = \{\sigma_U, \sigma_I, \zeta,\xi,\theta,\beta\}$ denote the parameters of the model, 
the posterior predictive on the test data has the form:
\begin{align}
\label{post_pred}
\int \Pr(G_\mathrm{test} \in \cdot \given \Theta ) \Pr(\Theta \given G_\mathrm{train} , G_\mathrm{holdoutfit}) \intd \Theta.
\end{align}

Running our inference algorithm on $G_{\mathrm{test}}$ returns a distribution 
$q(\sigma_U, \sigma_I, \xi_{\mathrm{train}}, \zeta, \theta_{\mathrm{train}}, \beta )$
that approximates $\Pr(\Theta \given G_{\mathrm{test}})$---the subscript $\mathrm{train}$ on the user parameters 
denotes restriction to only those users that are included in the training set.  
Intuitively speaking, we would like to approximate $\Pr(\Theta \given G_\mathrm{train} , G_\mathrm{holdoutfit}) \approx q(\Theta)$.

The difficulty is that $q$ does not carry any information about the users in the test set.
To resolve this, we introduce a new approximation:
\begin{multline*}
\Pr(\sigma_U, \sigma_I, \xi_{\mathrm{test}}, \zeta, \theta_{\mathrm{test}}, \beta \given G_{\mathrm{holdoutfit}})\\ \approx
q'(\xi_{\mathrm{test}}, \theta_{\mathrm{test}})q(\sigma_U, \sigma_I, \zeta, \beta). 
\end{multline*}
That is, we approximate the posterior distribution of parameters given $G_{\mathrm{holdoutfit}}$ as a
factorized distribution consisting of the approximate posterior over the sparsity and item parameters (already learned from the training data),
and a new approximating distribution $q'$ over the user parameters.
We take $q'$ in the same family as we would use for mean-field variational inference on $G_{\mathrm{holdoutfit}}$
and fit it by minimizing the variational inference loss between $\Pr(\sigma_U,
\sigma_I, \xi_{\mathrm{test}}, \zeta, \theta_{\mathrm{test}}, \beta \given
G_{\mathrm{holdoutfit}})$ and the distribution $q'(\xi_{\mathrm{test}}, \theta_{\mathrm{test}})q(\sigma_U, \sigma_I, \zeta, \beta)$.
In practice, this is achieved by running our coordinate ascent variational inference scheme using $G_{\mathrm{holdoutfit}}$ as the dataset,
and fixing $q(\sigma_U, \sigma_I, \zeta, \beta)$ to the distribution learned on the training set.

In summary, we take 
\begin{multline*}
\Pr(\sigma_U, \sigma_I, \xi_{\mathrm{test}}, \zeta, \theta_{\mathrm{test}}, \beta \given G_\mathrm{train}, G_{\mathrm{holdoutfit}})\\ \approx 
\delta_{\hat{\sigma}_U, \hat{\sigma}_I} q'(\xi_{\mathrm{test}}) q'(\theta_{\mathrm{test}}) q(\zeta) q(\beta).
\end{multline*}
The posterior predictive distribution for the test set is then computed by substituting this approximation into \cref{post_pred}.

There is one remaining subtlety: we must also set the sample sizes $\usize_{\mathrm{test}}$ and $\isize_{\mathrm{test}}$ for the posterior predictive distribution for the test set. It follows from the structure of the test--train split that $\usize_{\mathrm{test}} = \frac{p}{1-p}\usize_{\mathrm{train}}$
and that $\isize_{\mathrm{test}} = q \isize_{\mathrm{train}}$.

\subsubsection{Posterior Predictive Checks} 
We now assess the quality of the approximate predictive posterior distribution for the test data.
In principle, the performance depends on three separate levels of approximation:
\begin{enumerate}
\item whether an exchangeable model is appropriate (i.e., can the distribution on the test data be estimated in an unbiased way from the training data)
\item whether Poisson matrix factorization a suitable model, and
\item whether the various approximations used in the inference are sound 
\end{enumerate}
We assess the quality by drawing samples from the posterior predictive distribution (over the test set) 
and comparing summary statistics between these samples and the true test data.

\cref{fig:ppc} plots the degree distributions of the posterior samples from dense and sparse models against test datasets. 
\cref{tab:ppc} gives simple summary statistic for the posterior draws. 
The model does a good job predicting the structure of the test data.

For datasets that appear genuinely sparse---e.g., Echonest or Wikipedia---the predictive performance of the sparse models seems better.
The sparse model generates a large number of (low degree) vertices that are present in the actual data, but that are missed by the dense model.
$\hat{\sigma}$ estimated on the sparse model samples are closer to $\hat{\sigma}$ estimates on the actual data.
This can be interpreted as meaning either that the sparse models do a better job of capturing degree heterogeneity---recall $\hat{\sigma}$ is a 
function of the degree distribution---or simply that the sparse models do a better job of predicting sparsity in the data. 
It is an interesting fact that even dense Poisson matrix factorization is able to predict some degree of sparsity in the data,
although it biases towards increased density.
\begin{figure}[!t]
  \centering
  \subfloat[Simulated Dataset: On left is the posterior draw from sparse model and on right is the posterior draw from dense model]{
  \includegraphics[width=0.5\linewidth]{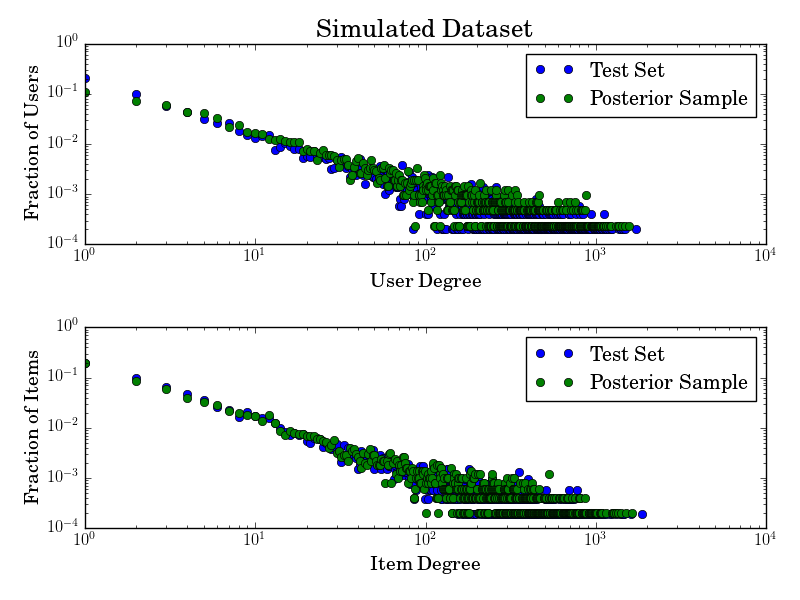}
  \includegraphics[width=0.5\linewidth]{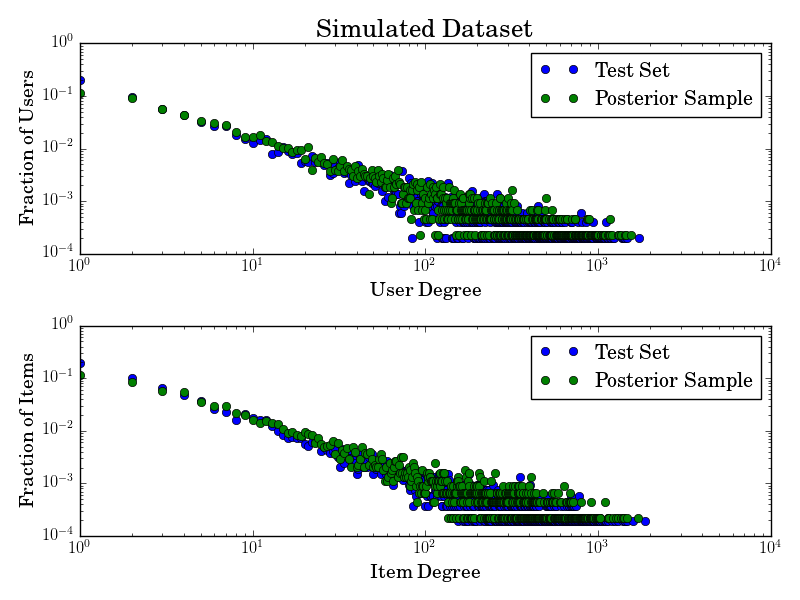}}
\\
  \subfloat[Wikipedia Dataset: On left is the posterior draw from sparse model and on right is the posterior draw from dense model]{
  \includegraphics[width=0.5\linewidth]{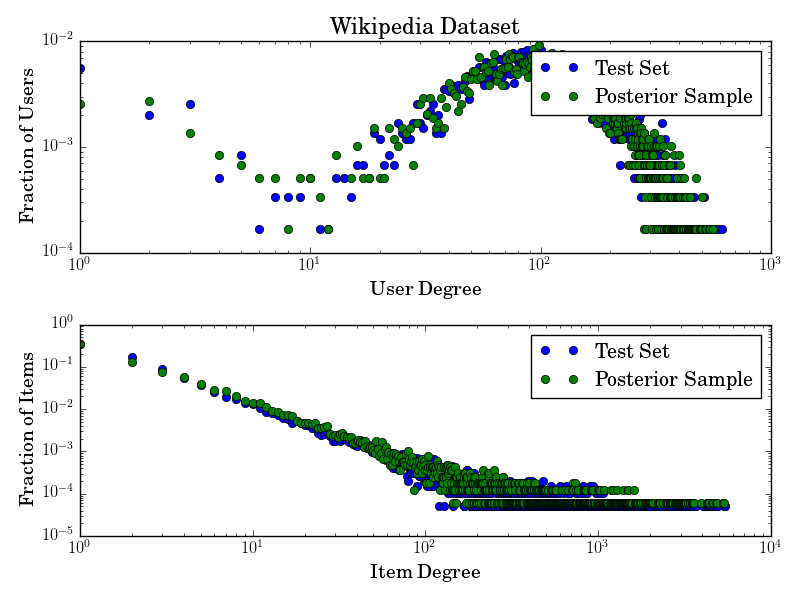}
  \includegraphics[width=0.5\linewidth]{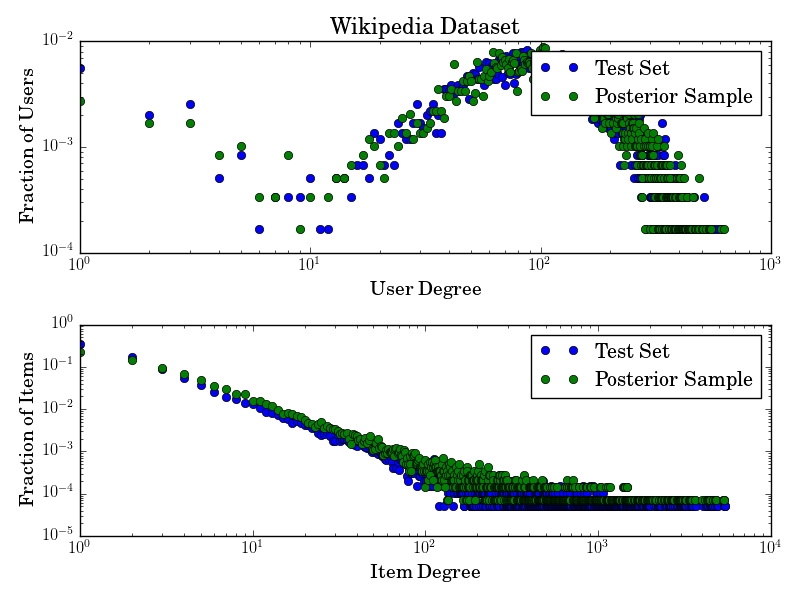}}
  
  \caption{Posterior Predictive Checks. 
    We compare samples from the approximate posterior predictive for the test data given the training data
    to the actual test data. The degree distributions are closely aligned, suggesting that
    the models do a good job predicting the structure of the test data. 
  \label{fig:ppc}}
\end{figure}

\begin{table}[!t]
  \centering
  \resizebox{\linewidth}{!}{
  \begin{tabular}{|c r|c|c|c|c|c|c|c|c|}
    \toprule
     &  & Simulated Data & Netflix  & Echonest & Wikipedia \\
    \midrule
          & Test Set:     & $5,029$ & $94,083$ & $199,655$ & $5,914$ \\
    $V_U$ & Sparse Model: & $4,319$ & $92,560$ & $189,971$ & $5,895$ \\
          & Dense Model:  & $4,230$ & $92,682$ & $190,617$ & $5,886$ \\
    \hline
          & Test Set:     & $5,285$ & $3,629$ & $59,270$ & $19,558$ \\
    $V_I$ & Sparse Model: & $4,906$ & $3,812$ & $56,272$ & $15,985$ \\
          & Dense Model:  & $4,474$ & $3,647$ & $51,151$ & $14,226$ \\
    \hline
          & Test Set:     & $374,155$ & $4,356,458$ & $1,895,016$ & $812,208$ \\
    $|E|$ & Sparse Model: & $375,897$ & $4,352,697$ & $1,870,403$ & $814,144$ \\
          & Dense Model:  & $377,313$ & $4,369,108$ & $1,882,812$ & $812,580$ \\
    \hline
                                        & Test Set:     & $(0.2262, 0.2298)$ & $(0.0122, 0.0015)$ & $(0.1671, 0.2535)$ & $(0.0097, 0.3886)$ \\
    $(\hat{\sigma_U}, \hat{\sigma_I})$  & Sparse Model: & $(0.1428, 0.2034)$ & $(0.0564, 0.0387)$ & $(0.1786, 0.2638)$ & $(0.0071, 0.3366)$ \\
                                        & Dense Model:  & $(0.1397, 0.1510)$ & $(0.0551, 0.0069)$ & $(0.1744, 0.1970)$ & $(0.0069, 0.2784)$ \\
    \bottomrule
  \end{tabular}
  }
  \caption{Posterior Predictive Checks.  We compare samples from the approximate posterior 
    predictive for the test data given the training data to the actual test data. 
    The sparse models do a moderately better job predicting the structure of the test set when the dataset is sparse.
    \label{tab:ppc}
}
\end{table}

\subsubsection{Recommendation Task.}
For each user in the test set, we recommend a list of items in the test set. 
To generate recommendations for each user $i$, we rank the items according to the values of
$r^i_j = \EE[\gamma_i \sum_k \theta_{ik} \beta_{jk} \omega_j]$; 
i.e., the expected number of edges between the user and the item.
This is a computationally cheap proxy for the probability of an edge between the user and the item. 

We measure the quality of the recommendations using four measures. 
The first is the top $M=20$ recall from \cite{Gopalan:Hofman:Blei:2015}, where $\mathrm{score} = \frac{|\{\texttt{Items retrieved}\}\cap \{\texttt{Items in test}\}|}{\mathrm{min}(M,|\{\texttt{Items in test}\}|)}.$
The second is top $M=20$ recall where recommendations are made only for unpopular items; i.e. we do not recommend items with degree more than a pre-specified threshold. 
The third is normalized Discounted Cumulative Gain (nDCG), a common measure of ranking quality. 
Finally, we consider unpopular nDCG where recommendations are made only for unpopular items.

Table \ref{tab:recommendation} summarizes the results of these experiments. 
We observe no difference in recommendation performance between the dense and sparse models, even on sparse datasets.

\begin{table}[!t]
  \centering
  \resizebox{\linewidth}{!}{
  \begin{tabular}{|c r|c|c|c|c|c|c|c|c|}
    \toprule
    Experiment& & Simulated  & Netflix  & Echonest & Wikipedia \\
    \midrule

Top $20$ & Sparse Model:                 & $0.4566$ & $0.5356$ & $0.1618$ & $0.8238$ \\
          & Dense Model:                & $0.4544$ & $0.5311$ & $0.1619$ & $0.8236$ \\
          \hline
Top $20$ items, unpopular & Sparse Model: & $0.2820$ & $0.3019$ & $0.0410$ & $0.1031$ \\
          & Dense Model:                & $0.2871$ & $0.3024$ & $0.0421$ & $0.1043$ \\
          \hline
Normalized DCG & Sparse Model:          & $0.5287$ & $0.6790$ & $0.2661$ & $0.7978$ \\
          & Dense Model:                & $0.5284$ & $0.6779$ & $0.2633$ & $0.7993$ \\
          \hline
Normalized DCG, unpopular & Sparse Model:          & $0.3900$ & $0.4374$ & $0.1019$ & $0.2346$ \\
          & Dense Model:                & $0.3908$ & $0.4386$ & $0.1040$ & $0.2403$ \\
    \bottomrule
  \end{tabular}
  }
  \caption{Scores on the recommendation task according to various measures. 
Unpopular rows report scoring on recommendations excluding the 5\% highest degree items.  
There is no discernible difference in performance between 
    the dense and sparse models.}
  \label{tab:recommendation}
\end{table}

\section{Discussion}

The main insights from this paper are: First, sparsity has a signature that can be easily recognized in fixed size datasets, 
and sparse behaviour occurs in real-world data. 
Second, the (common) assumption that the data is generated 
according to an exchangeable model implies a unique correct scheme for splitting the data into test and train sets, 
and this splitting procedure can be adapted as a component of practical model evaluation procedures. 
Finally, it is possible to scale inference for sparse exchangeable models to very large datasets.

An intriguing question raised by this paper is why modelling sparsity does not seem to help with recommendation performance,
even in cases where the dataset is clearly sparse.
One possible explanation is that Poisson matrix factorization is particularly robust against the sparsity misspecification; see \cite{Zhou:2017}
for a discussion of this point. In this case, we would expect to see a performance difference in more powerful models.
Another possible explanation is that accounting for sparsity gives more modelling power in a way that is generally not relevant for recommendation.

An obvious direction for future work is to establish a wider range of practical sparse exchangeable models. 
The tools developed in this paper will be generally useful for this enterprise,
particularly for sparse graph models built on the Generalized Gamma Process.

\bibliographystyle{ACM-Reference-Format}
\bibliography{random_graphs,matrix_factorization}


\begin{thebibliography}{24}


\ifx \showCODEN    \undefined \def \showCODEN     #1{\unskip}     \fi
\ifx \showDOI      \undefined \def \showDOI       #1{#1}\fi
\ifx \showISBNx    \undefined \def \showISBNx     #1{\unskip}     \fi
\ifx \showISBNxiii \undefined \def \showISBNxiii  #1{\unskip}     \fi
\ifx \showISSN     \undefined \def \showISSN      #1{\unskip}     \fi
\ifx \showLCCN     \undefined \def \showLCCN      #1{\unskip}     \fi
\ifx \shownote     \undefined \def \shownote      #1{#1}          \fi
\ifx \showarticletitle \undefined \def \showarticletitle #1{#1}   \fi
\ifx \showURL      \undefined \def \showURL       {\relax}        \fi
\providecommand\bibfield[2]{#2}
\providecommand\bibinfo[2]{#2}
\providecommand\natexlab[1]{#1}
\providecommand\showeprint[2][]{arXiv:#2}

\bibitem[\protect\citeauthoryear{Abadi, Agarwal, Barham, Brevdo, Chen, Citro,
  Corrado, Davis, Dean, Devin, Ghemawat, Goodfellow, Harp, Irving, Isard, Jia,
  Jozefowicz, Kaiser, Kudlur, Levenberg, Man\'{e}, Monga, Moore, Murray, Olah,
  Schuster, Shlens, Steiner, Sutskever, Talwar, Tucker, Vanhoucke, Vasudevan,
  Vi\'{e}gas, Vinyals, Warden, Wattenberg, Wicke, Yu, and Zheng}{Abadi
  et~al\mbox{.}}{2015}]%
        {tensorflow2015-whitepaper}
\bibfield{author}{\bibinfo{person}{Mart\'{\i}n Abadi}, \bibinfo{person}{Ashish
  Agarwal}, \bibinfo{person}{Paul Barham}, \bibinfo{person}{Eugene Brevdo},
  \bibinfo{person}{Zhifeng Chen}, \bibinfo{person}{Craig Citro},
  \bibinfo{person}{Greg~S. Corrado}, \bibinfo{person}{Andy Davis},
  \bibinfo{person}{Jeffrey Dean}, \bibinfo{person}{Matthieu Devin},
  \bibinfo{person}{Sanjay Ghemawat}, \bibinfo{person}{Ian Goodfellow},
  \bibinfo{person}{Andrew Harp}, \bibinfo{person}{Geoffrey Irving},
  \bibinfo{person}{Michael Isard}, \bibinfo{person}{Yangqing Jia},
  \bibinfo{person}{Rafal Jozefowicz}, \bibinfo{person}{Lukasz Kaiser},
  \bibinfo{person}{Manjunath Kudlur}, \bibinfo{person}{Josh Levenberg},
  \bibinfo{person}{Dan Man\'{e}}, \bibinfo{person}{Rajat Monga},
  \bibinfo{person}{Sherry Moore}, \bibinfo{person}{Derek Murray},
  \bibinfo{person}{Chris Olah}, \bibinfo{person}{Mike Schuster},
  \bibinfo{person}{Jonathon Shlens}, \bibinfo{person}{Benoit Steiner},
  \bibinfo{person}{Ilya Sutskever}, \bibinfo{person}{Kunal Talwar},
  \bibinfo{person}{Paul Tucker}, \bibinfo{person}{Vincent Vanhoucke},
  \bibinfo{person}{Vijay Vasudevan}, \bibinfo{person}{Fernanda Vi\'{e}gas},
  \bibinfo{person}{Oriol Vinyals}, \bibinfo{person}{Pete Warden},
  \bibinfo{person}{Martin Wattenberg}, \bibinfo{person}{Martin Wicke},
  \bibinfo{person}{Yuan Yu}, {and} \bibinfo{person}{Xiaoqiang Zheng}.}
  \bibinfo{year}{2015}\natexlab{}.
\newblock \bibinfo{title}{{TensorFlow}: Large-Scale Machine Learning on
  Heterogeneous Systems}.
\newblock   (\bibinfo{year}{2015}).
\newblock
\showURL{%
\url{https://www.tensorflow.org/}}
\newblock
\shownote{Software available from tensorflow.org.}


\bibitem[\protect\citeauthoryear{Aldous}{Aldous}{1981}]%
        {Aldous:1981}
\bibfield{author}{\bibinfo{person}{David~J. Aldous}.}
  \bibinfo{year}{1981}\natexlab{}.
\newblock \showarticletitle{Representations for partially exchangeable arrays
  of random variables}.
\newblock \bibinfo{journal}{{\em J. Multivariate Anal.\/}}
  \bibinfo{volume}{11}, \bibinfo{number}{4} (\bibinfo{year}{1981}),
  \bibinfo{pages}{581--598}.
\newblock
\showCODEN{JMVAAI}
\showISSN{0047-259X}
\showDOI{%
\url{https://doi.org/10.1016/0047-259X(81)90099-3}}


\bibitem[\protect\citeauthoryear{Bertin-Mahieux, Ellis, Whitman, and
  Lamere}{Bertin-Mahieux et~al\mbox{.}}{2011}]%
        {Bertin-Mahieux:Ellis:Whitman:Lamere:2011}
\bibfield{author}{\bibinfo{person}{Thierry Bertin-Mahieux},
  \bibinfo{person}{Daniel~P.W. Ellis}, \bibinfo{person}{Brian Whitman}, {and}
  \bibinfo{person}{Paul Lamere}.} \bibinfo{year}{2011}\natexlab{}.
\newblock \showarticletitle{The Million Song Dataset}. In
  \bibinfo{booktitle}{{\em {Proceedings of the 12th International Conference on
  Music Information Retrieval ({ISMIR} 2011)}}}.
\newblock


\bibitem[\protect\citeauthoryear{{Blei}, {Kucukelbir}, and {McAuliffe}}{{Blei}
  et~al\mbox{.}}{2016}]%
        {Blei:Kucukelbir:McAuliffe:2016}
\bibfield{author}{\bibinfo{person}{D.~M. {Blei}}, \bibinfo{person}{A.
  {Kucukelbir}}, {and} \bibinfo{person}{J.~D. {McAuliffe}}.}
  \bibinfo{year}{2016}\natexlab{}.
\newblock \showarticletitle{{Variational Inference: A Review for
  Statisticians}}.
\newblock \bibinfo{journal}{{\em ArXiv e-prints\/}} (\bibinfo{date}{Jan.}
  \bibinfo{year}{2016}).
\newblock
\showeprint[arxiv]{stat.CO/1601.00670}


\bibitem[\protect\citeauthoryear{Borgs, Chayes, Cohn, and Veitch}{Borgs
  et~al\mbox{.}}{2017}]%
        {Borgs:Chayes:Cohn:Veitch:2017}
\bibfield{author}{\bibinfo{person}{C. Borgs}, \bibinfo{person}{J. Chayes},
  \bibinfo{person}{H. Cohn}, {and} \bibinfo{person}{V. Veitch}.}
  \bibinfo{year}{2017}\natexlab{}.
\newblock \bibinfo{title}{Sampling perspectives on sparse exchangeable graphs}.
\newblock \bibinfo{howpublished}{Preprint}.   (\bibinfo{year}{2017}).
\newblock


\bibitem[\protect\citeauthoryear{{Borgs}, {Chayes}, {Cohn}, and
  {Holden}}{{Borgs} et~al\mbox{.}}{2016}]%
        {Borgs:Chayes:Cohn:Holden:2016}
\bibfield{author}{\bibinfo{person}{C. {Borgs}}, \bibinfo{person}{J.~T.
  {Chayes}}, \bibinfo{person}{H. {Cohn}}, {and} \bibinfo{person}{N. {Holden}}.}
  \bibinfo{year}{2016}\natexlab{}.
\newblock \showarticletitle{{Sparse exchangeable graphs and their limits via
  graphon processes}}.
\newblock \bibinfo{journal}{{\em ArXiv e-prints\/}} (\bibinfo{date}{1}
  \bibinfo{year}{2016}).
\newblock
\showeprint[arxiv]{math.PR/1601.07134}


\bibitem[\protect\citeauthoryear{Caron}{Caron}{2012}]%
        {Caron:2012}
\bibfield{author}{\bibinfo{person}{Francois Caron}.}
  \bibinfo{year}{2012}\natexlab{}.
\newblock \showarticletitle{Bayesian nonparametric models for bipartite
  graphs}.
\newblock In \bibinfo{booktitle}{{\em Advances in Neural Information Processing
  Systems 25}}, \bibfield{editor}{\bibinfo{person}{F.~Pereira},
  \bibinfo{person}{C.~J.~C. Burges}, \bibinfo{person}{L.~Bottou}, {and}
  \bibinfo{person}{K.~Q. Weinberger}} (Eds.). \bibinfo{publisher}{Curran
  Associates, Inc.}, \bibinfo{pages}{2051--2059}.
\newblock
\showURL{%
\url{http://papers.nips.cc/paper/4837-bayesian-nonparametric-models-for-bipartite-graphs.pdf}}


\bibitem[\protect\citeauthoryear{Caron and Fox}{Caron and Fox}{2017}]%
        {Caron:Fox:2017}
\bibfield{author}{\bibinfo{person}{François Caron} {and}
  \bibinfo{person}{Emily~B. Fox}.} \bibinfo{year}{2017}\natexlab{}.
\newblock \showarticletitle{Sparse graphs using exchangeable random measures}.
\newblock \bibinfo{journal}{{\em Journal of the Royal Statistical Society:
  Series B (Statistical Methodology)\/}} \bibinfo{volume}{79},
  \bibinfo{number}{5} (\bibinfo{year}{2017}), \bibinfo{pages}{1295--1366}.
\newblock
\showISSN{1467-9868}
\showDOI{%
\url{https://doi.org/10.1111/rssb.12233}}


\bibitem[\protect\citeauthoryear{Caron and Rousseau}{Caron and
  Rousseau}{2017}]%
        {Caron:Rousseau:2017}
\bibfield{author}{\bibinfo{person}{Fran{\c{c}}ois Caron} {and}
  \bibinfo{person}{Judith Rousseau}.} \bibinfo{year}{2017}\natexlab{}.
\newblock \showarticletitle{On sparsity and power-law properties of graphs
  based on exchangeable point processes}.
\newblock \bibinfo{journal}{{\em arXiv preprint arXiv:1708.03120\/}}
  (\bibinfo{year}{2017}).
\newblock


\bibitem[\protect\citeauthoryear{Ghahramani and Beal}{Ghahramani and
  Beal}{2001}]%
        {Ghahramani:Beal:2001}
\bibfield{author}{\bibinfo{person}{Zoubin Ghahramani} {and}
  \bibinfo{person}{Matthew~J Beal}.} \bibinfo{year}{2001}\natexlab{}.
\newblock \showarticletitle{Propagation algorithms for variational Bayesian
  learning}. In \bibinfo{booktitle}{{\em Advances in neural information
  processing systems}}. \bibinfo{pages}{507--513}.
\newblock


\bibitem[\protect\citeauthoryear{Gopalan, Hofman, and Blei}{Gopalan
  et~al\mbox{.}}{2015}]%
        {Gopalan:Hofman:Blei:2015}
\bibfield{author}{\bibinfo{person}{Prem Gopalan}, \bibinfo{person}{Jake~M
  Hofman}, {and} \bibinfo{person}{David~M Blei}.}
  \bibinfo{year}{2015}\natexlab{}.
\newblock \showarticletitle{Scalable Recommendation with Hierarchical Poisson
  Factorization.}. In \bibinfo{booktitle}{{\em UAI}}.
  \bibinfo{pages}{326--335}.
\newblock


\bibitem[\protect\citeauthoryear{Herlau, Schmidt, and M{\o}rup}{Herlau
  et~al\mbox{.}}{2016}]%
        {Herlau:Schmidt:Morup:2016}
\bibfield{author}{\bibinfo{person}{Tue Herlau}, \bibinfo{person}{Mikkel~N
  Schmidt}, {and} \bibinfo{person}{Morten M{\o}rup}.}
  \bibinfo{year}{2016}\natexlab{}.
\newblock \showarticletitle{Completely random measures for modelling
  block-structured sparse networks}.
\newblock In \bibinfo{booktitle}{{\em Advances in Neural Information Processing
  Systems 29}}, \bibfield{editor}{\bibinfo{person}{D.~D. Lee},
  \bibinfo{person}{M.~Sugiyama}, \bibinfo{person}{U.~V. Luxburg},
  \bibinfo{person}{I.~Guyon}, {and} \bibinfo{person}{R.~Garnett}} (Eds.).
  \bibinfo{publisher}{Curran Associates, Inc.}, \bibinfo{pages}{4260--4268}.
\newblock
\showURL{%
\url{http://papers.nips.cc/paper/6521-completely-random-measures-for-modelling-block-structured-sparse-networks.pdf}}


\bibitem[\protect\citeauthoryear{Hoffman, Blei, and Cook}{Hoffman
  et~al\mbox{.}}{2010}]%
        {Hoffman:Blei:Cook:2010}
\bibfield{author}{\bibinfo{person}{Matthew~D. Hoffman},
  \bibinfo{person}{David~M. Blei}, {and} \bibinfo{person}{Perry~R. Cook}.}
  \bibinfo{year}{2010}\natexlab{}.
\newblock \showarticletitle{Bayesian Nonparametric Matrix Factorization for
  Recorded Music}. In \bibinfo{booktitle}{{\em Proceedings of the 27th
  International Conference on International Conference on Machine Learning}}
  {\em (\bibinfo{series}{ICML'10})}. \bibinfo{publisher}{Omnipress},
  \bibinfo{address}{USA}, \bibinfo{pages}{439--446}.
\newblock
\showISBNx{978-1-60558-907-7}
\showURL{%
\url{http://dl.acm.org/citation.cfm?id=3104322.3104379}}


\bibitem[\protect\citeauthoryear{Hoffman, Blei, Wang, and Paisley}{Hoffman
  et~al\mbox{.}}{2013}]%
        {Hoffman:Blei:Wang:Paisley:2013}
\bibfield{author}{\bibinfo{person}{Matthew~D Hoffman}, \bibinfo{person}{David~M
  Blei}, \bibinfo{person}{Chong Wang}, {and} \bibinfo{person}{John~William
  Paisley}.} \bibinfo{year}{2013}\natexlab{}.
\newblock \showarticletitle{Stochastic variational inference.}
\newblock \bibinfo{journal}{{\em Journal of Machine Learning Research\/}}
  \bibinfo{volume}{14}, \bibinfo{number}{1} (\bibinfo{year}{2013}),
  \bibinfo{pages}{1303--1347}.
\newblock


\bibitem[\protect\citeauthoryear{Hoover}{Hoover}{1979}]%
        {Hoover:1979}
\bibfield{author}{\bibinfo{person}{D.~N. Hoover}.}
  \bibinfo{year}{1979}\natexlab{}.
\newblock \bibinfo{booktitle}{{\em Relations on probability spaces and arrays
  of random variables}}.
\newblock \bibinfo{type}{{T}echnical {R}eport}. \bibinfo{institution}{Institute
  of Advanced Study, Princeton}.
\newblock


\bibitem[\protect\citeauthoryear{Janson}{Janson}{2016}]%
        {Janson:2016}
\bibfield{author}{\bibinfo{person}{S. Janson}.}
  \bibinfo{year}{2016}\natexlab{}.
\newblock \showarticletitle{{Graphons and cut metric on sigma-finite measure
  spaces}}.
\newblock \bibinfo{journal}{{\em ArXiv e-prints\/}} (\bibinfo{date}{8}
  \bibinfo{year}{2016}).
\newblock
\showeprint[arxiv]{math.CO/1608.01833}


\bibitem[\protect\citeauthoryear{Janson}{Janson}{2017}]%
        {Janson:2017}
\bibfield{author}{\bibinfo{person}{S. Janson}.}
  \bibinfo{year}{2017}\natexlab{}.
\newblock \showarticletitle{{On convergence for graphexes}}.
\newblock \bibinfo{journal}{{\em ArXiv e-prints\/}} (\bibinfo{date}{2}
  \bibinfo{year}{2017}).
\newblock
\showeprint[arxiv]{math.PR/1702.06389}


\bibitem[\protect\citeauthoryear{Merity, Xiong, Bradbury, and Socher}{Merity
  et~al\mbox{.}}{2016}]%
        {Merity:Xiong:Bradbury:Socher:2016}
\bibfield{author}{\bibinfo{person}{Stephen Merity}, \bibinfo{person}{Caiming
  Xiong}, \bibinfo{person}{James Bradbury}, {and} \bibinfo{person}{Richard
  Socher}.} \bibinfo{year}{2016}\natexlab{}.
\newblock \showarticletitle{Pointer Sentinel Mixture Models}.
\newblock \bibinfo{journal}{{\em CoRR\/}}  \bibinfo{volume}{abs/1609.07843}
  (\bibinfo{year}{2016}).
\newblock
\showeprint[arxiv]{1609.07843}
\showURL{%
\url{http://arxiv.org/abs/1609.07843}}


\bibitem[\protect\citeauthoryear{Naulet, Sharma, Veitch, and Roy}{Naulet
  et~al\mbox{.}}{2017}]%
        {Naulet:Sharma:Veitch:Roy:2017}
\bibfield{author}{\bibinfo{person}{Z. Naulet}, \bibinfo{person}{E. Sharma},
  \bibinfo{person}{V. Veitch}, {and} \bibinfo{person}{D.M. Roy}.}
  \bibinfo{year}{2017}\natexlab{}.
\newblock \showarticletitle{{An Estimator for the Tail-Index of Graphex
  Processes}}.
\newblock \bibinfo{journal}{{\em ArXiv e-prints\/}} (\bibinfo{date}{11}
  \bibinfo{year}{2017}).
\newblock
\showeprint[arxiv]{math.ST/1712.01745}


\bibitem[\protect\citeauthoryear{Orbanz and Roy}{Orbanz and Roy}{2015}]%
        {Orbanz:Roy:2015}
\bibfield{author}{\bibinfo{person}{P. Orbanz} {and} \bibinfo{person}{D.M.
  Roy}.} \bibinfo{year}{2015}\natexlab{}.
\newblock \showarticletitle{Bayesian Models of Graphs, Arrays and Other
  Exchangeable Random Structures}.
\newblock \bibinfo{journal}{{\em Pattern Analysis and Machine Intelligence,
  IEEE Transactions on\/}} \bibinfo{volume}{37}, \bibinfo{number}{2}
  (\bibinfo{date}{2} \bibinfo{year}{2015}), \bibinfo{pages}{437--461}.
\newblock
\showISSN{0162-8828}
\showDOI{%
\url{https://doi.org/10.1109/TPAMI.2014.2334607}}


\bibitem[\protect\citeauthoryear{Todeschini and Caron}{Todeschini and
  Caron}{2016}]%
        {Todeschini:Caron:2016}
\bibfield{author}{\bibinfo{person}{A. Todeschini} {and} \bibinfo{person}{F.
  Caron}.} \bibinfo{year}{2016}\natexlab{}.
\newblock \showarticletitle{{Exchangeable Random Measures for Sparse and
  Modular Graphs with Overlapping Communities}}.
\newblock \bibinfo{journal}{{\em ArXiv e-prints\/}} (\bibinfo{date}{2}
  \bibinfo{year}{2016}).
\newblock
\showeprint[arxiv]{stat.ME/1602.02114}


\bibitem[\protect\citeauthoryear{{Veitch} and {Roy}}{{Veitch} and
  {Roy}}{2015}]%
        {Veitch:Roy:2015}
\bibfield{author}{\bibinfo{person}{V. {Veitch}} {and} \bibinfo{person}{D.~M.
  {Roy}}.} \bibinfo{year}{2015}\natexlab{}.
\newblock \showarticletitle{The Class of Random Graphs Arising from
  Exchangeable Random Measures}.
\newblock \bibinfo{journal}{{\em ArXiv e-prints\/}} (\bibinfo{date}{12}
  \bibinfo{year}{2015}).
\newblock
\showeprint[arxiv]{math.ST/1512.03099}


\bibitem[\protect\citeauthoryear{{Veitch} and {Roy}}{{Veitch} and
  {Roy}}{2016}]%
        {Veitch:Roy:2016}
\bibfield{author}{\bibinfo{person}{V. {Veitch}} {and} \bibinfo{person}{D.~M.
  {Roy}}.} \bibinfo{year}{2016}\natexlab{}.
\newblock \showarticletitle{{Sampling and Estimation for (Sparse) Exchangeable
  Graphs}}.
\newblock \bibinfo{journal}{{\em ArXiv e-prints\/}} (\bibinfo{date}{11}
  \bibinfo{year}{2016}).
\newblock
\showeprint[arxiv]{math.ST/1611.00843}


\bibitem[\protect\citeauthoryear{Zhou}{Zhou}{2017}]%
        {Zhou:2017}
\bibfield{author}{\bibinfo{person}{M. Zhou}.} \bibinfo{year}{2017}\natexlab{}.
\newblock \showarticletitle{Discussion on “Sparse graphs using exchangeable
  random measures” by Francois Caron and Emily B. Fox}.
\newblock
  \bibinfo{howpublished}{\url{https://mingyuanzhou.github.io/Papers/Zhou_Discussion_CaronFox_JRSSB.pdf}}.
\newblock  (\bibinfo{year}{2017}).
\newblock


\end{thebibliography}

\medskip

\appendix

\section{Complete conditional for leftover mass}
\label{sec:leftover_cond}
Here we sketch the computation for the complete conditional mean and variance of 
the total leftover item masses $(\urm^{\usize}_1, \dots, \urm^{\usize}_K)$.

We may view $\{(\gamma_i, \theta_{i1}, \dots, \theta_{iK})\}$ as a Poisson process
on $\NNReals \times \NNReals^K$ with mean measure
\begin{equation*}
\usize g_{\sigma_U,\tau_U}(\intd \gamma) F(\intd\theta_1, \dots, \intd \theta_K), 
\end{equation*}
where $F(\intd\theta_1, \dots, \intd \theta_K) = \prod_k \gammaDist(\theta_k; a,b)$.
This follows by projecting the GGP on $\NNReals \times [0,\usize)$ onto the feature space,
and viewing $\{(\gamma_i, \theta_{i1}, \dots, \theta_{iK})\}$ as a marking.

Conditional on $\zeta$ and $\beta$, each atom $(\gamma_i, \theta_{i1}, \dots, \theta_{iK})$ of this process connects to $0$ edges
independently with probability 
\begin{equation*}
\exp(-\gamma_i \sum_k \theta_{ik} \sum_j \omega_j \beta_{jk}). 
\end{equation*}
Thus, ``connects to no edges'' may be viewed as a marking of the point process. 
It then follows that the random set of atoms that fail to connect
to any edges is a point process $\{(\gamma^*_i, \theta^*_{i1}, \dots, \theta^*_{iK})\}$ with mean measure
\begin{equation*}
\usize g_{\sigma_U,\tau_U}(\intd \gamma^*) F(\intd\theta^*_1, \dots, \intd \theta^*_K) \exp(-\gamma^* \sum_k \theta^*_{k} \sum_j \omega_j \beta_{jk}).
\end{equation*}
Notice that $\urm^{\usize}_k = \sum_i \gamma^*_i \theta^*_{ik}$ for each $k$. The claimed result then follows by Campbell's theorem and some algebraic manipulation.

\end{document}

